\documentclass[twoside,11pt]{article}
\usepackage{fullpage}
\usepackage{amssymb, amsmath}
\usepackage{graphicx}
\usepackage{bbm}
\usepackage{multirow}

\newcommand{\argmin}{\operatorname*{arg \ min}}

\usepackage{algorithm,algorithmic}
\usepackage{setspace}
\usepackage{multicol}
\usepackage{blindtext}
\usepackage[table]{xcolor}
\usepackage{xr}
\usepackage{booktabs}
\usepackage{breqn}
\usepackage{natbib}
\usepackage{tikz}
\usepackage{subfigure}
\usepackage{array}
\usepackage{enumerate}
\usepackage[shortlabels]{enumitem}
\newcolumntype{C}[1]{>{\centering\arraybackslash}p{#1}}
\usetikzlibrary{shapes.geometric}
\usetikzlibrary{arrows,automata}
\usepackage{amsmath}
\usepackage{amsthm}
\usepackage{soul}
\usepackage{url}

\newtheorem{lemma}{Lemma}

\DeclareMathOperator{\SE}{SE}
\DeclareMathOperator{\E}{E}
\DeclareMathOperator{\Var}{Var}
\DeclareMathOperator{\Cov}{Cov}
\DeclareMathOperator{\median}{median}
\DeclareMathOperator{\sd}{SD}
\DeclareMathOperator{\se}{SE}

\usepackage{tcolorbox}

\newtcolorbox[auto counter]{mybox}[2][]{float,title={\textcolor{black}{Box~\thetcbcounter: #2}},#1, 
colframe=white!70!gray}

\title{Revisiting inference after prediction}
\author{Keshav Motwani \\
    Department of Biostatistics, University of Washington, \\ \\
    Daniela Witten \\
    Departments of Biostatistics and Statistics, University of Washington}
\date{\today}{}

\begin{document}

\maketitle

\begin{abstract}
Recent work has focused on the very common practice of prediction-based inference: that is,   (i) using a pre-trained machine learning model to predict an unobserved
response variable, and then (ii) conducting inference on the association between that predicted response and some covariates. 
As pointed out by \citet{wang2020methods}, applying a standard inferential approach in  (ii) does not accurately quantify the association between the unobserved (as opposed to the predicted) response and the covariates. In recent work, \citet{wang2020methods} and \citet{angelopoulos2023prediction} propose corrections to step (ii) in order to enable valid inference on the association between the unobserved response and the covariates. Here, we show that the method proposed by \citet{angelopoulos2023prediction} successfully controls the type 1 error rate and provides  confidence intervals with correct nominal coverage, regardless of the quality of the pre-trained machine learning model used to predict the unobserved response. However, the method proposed by \citet{wang2020methods} provides valid inference only under very strong conditions that rarely hold in practice: for instance,  
 if the machine learning model perfectly estimates the true regression function in the study population of interest.
\end{abstract}

\section{Introduction}

Rapid recent progress in the field of machine learning has enabled the development of complex and high-quality machine learning models to predict a response variable of interest.
 This is particularly attractive in settings where future measurement of this response variable is prohibitively expensive or impossible.  For example, instead of performing expensive experiments to determine a protein's structure, it is possible to obtain high-quality structural predictions  using AlphaFold \citep{jumper2021highly}. Similarly, in developing countries, determining the true cause of death may be impossible; instead, one might predict the cause of death on the basis of a ``verbal autopsy" \citep{clark2015insilicova,khoury1999mortality}. In the context of gene expression data, it is infeasible to conduct experiments in every possible tissue type, and so instead a machine learning model can be applied to predict gene expression in a tissue type of interest \citep{ellis2018improving,gamazon2015gene,gusev2019transcriptome}.  

 In this paper, we will use the notation $\hat{f}: \mathcal{Z} \rightarrow \mathcal{Y}$ to denote a pre-trained machine learning model that maps from $\mathcal{Z}$, the space of the predictors $Z$, to $\mathcal{Y}$, the space of the response variable $Y$. We assume that the operating characteristics of $\hat{f}(\cdot)$ in the study population of interest are unknown to the end-user, and the data used to fit $\hat{f}(\cdot)$ are  unavailable. Thus, in what follows, we will treat $\hat{f}(\cdot)$ as a ``black box" function.

In important recent papers, \citet{wang2020methods}  and \citet{angelopoulos2023prediction}
consider the practice of \emph{prediction-based inference}\footnote{\citet{wang2020methods}   and \citet{angelopoulos2023prediction} refer to this practice as
\emph{post-prediction inference} and 
  \emph{prediction-powered inference}, respectively; to unify terminology, here we refer to it as \emph{prediction-based inference}.}: that is, of quantifying the association between the response $Y$ and some covariates $X$ using realizations not of $(Y, X)$, but rather, of  $(\hat{f}(Z), X)$.  Such an approach is attractive in cases where the association between $Y$ and $X$ is of interest, and both $\hat{f}(\cdot)$ and a large sample of realizations of predictors $Z$ and covariates $X$ are available, but realizations of $Y$ are expensive or otherwise unavailable. Throughout, we will use $Z$ to denote the \emph{predictors} in the machine learning model $\hat{f}(\cdot)$, and $X$ to denote the \emph{covariates} whose association with $Y$ is of interest.  In many settings, the covariate variables $X$ may be a subset of the predictor variables $Z$, or may be identical to $Z$, but this is not necessarily the case.

\cite{wang2020methods} point out that 
 a naive approach to prediction-based inference that simplistically interprets the association between $(\hat{f}(Z), X)$ as the association between $(Y, X)$ is problematic from a statistical perspective. For instance, regressing $\hat{f}(Z)$ onto $X$ using least squares does not lead to valid inference on the association between $Y$ and $X$:
  e.g. it leads to hypothesis tests that fail to control type 1 error, and confidence intervals that do not achieve the nominal coverage. See Box 1.

\begin{mybox}[floatplacement=!h,label={box:naive}]{\emph{Naive approach to prediction-based inference.} 
We are given a (pre-trained) prediction function $\hat{f}(\cdot): \mathcal{Z} \rightarrow \mathcal{Y}$, and an unlabeled dataset $\mathbf{z}_{\text{unlab}}$ representing realizations from $Z$. 
As pointed out by \citet{wang2020methods}, the naive approach  displayed here does not allow for valid inference on the association between $Y$ and $X$.}
\begin{enumerate}[start=1,label={Step \arabic*:},leftmargin=*]
\item Compute $\hat{\mathbf{y}}_{\text{unlab}} = \hat{f}(\mathbf{z}_{\text{unlab}})$. 
\item Conduct inference on the association between $Y$ and $X$, using  $\hat{\mathbf{y}}_{\text{unlab}}$ and $\mathbf{x}_{\text{unlab}}$ as data, without accounting for the fact that $(\hat{\mathbf{y}}_{\text{unlab}}, \mathbf{x}_{\text{unlab}})$ is not a sample from the distribution of $(Y, X)$.
  \end{enumerate}
\end{mybox}

\citet{wang2020methods} and \citet{angelopoulos2023prediction} propose creative solutions to overcome this issue. 
They assume that in addition to a large unlabeled dataset $(\mathbf{z}_{\text{unlab}}, \mathbf{x}_{\text{unlab}})$, they also have access to a (relatively small) labeled dataset 
$(\mathbf{y}_{\text{lab}}, \mathbf{z}_{\text{lab}}, \mathbf{x}_{\text{lab}})$. We focus on the case where the labeled and unlabeled data are independent and identically distributed samples from the same study population of interest, though \citet{angelopoulos2023prediction} also extend beyond this setting. See Box 2. We emphasize that the goal is to quantify association between $Y$ and $X$.

\begin{mybox}[floatplacement=!h,label={box:goal}]{\emph{Setting of prediction-based inference.}} 
\begin{itemize}[leftmargin=*]
\item \emph{Given:} A pre-trained machine learning model $\hat{f}: \mathcal{Z} \rightarrow \mathcal{Y}$. 
\item \emph{Goal:} To quantify the association between a response $Y \in \mathcal{Y}$ and covariates $X \in \mathcal{X}$. 
\item \emph{Data:} A (relatively small) labeled dataset $(\mathbf{y}_{\text{lab}}, \mathbf{z}_{\text{lab}}, \mathbf{x}_{\text{lab}})$ consisting of i.i.d. realizations of $(Y, Z, X)$,  and a (large) unlabeled dataset $(\mathbf{z}_{\text{unlab}}, \mathbf{x}_{\text{unlab}})$ consisting of i.i.d. realizations of $(Z, X)$. Both are drawn from the same study population. 
  \end{itemize}
\end{mybox}

One simple (and valid) option is to quantify association between $Y$ and $X$ using only the labeled data
$(\mathbf{y}_{\text{lab}}, \mathbf{x}_{\text{lab}})$. However, this approach  entirely discards the vast amount of unlabeled data $(\mathbf{z}_{\text{unlab}}, \mathbf{x}_{\text{unlab}})$. Intuitively, if the prediction function $\hat{f}(\cdot)$ is nearly perfect on our population of interest,  then using 
$(\hat{f}(\mathbf{z}_{\text{unlab}}), \mathbf{x}_{\text{unlab}})$ in addition to $(\mathbf{y}_{\text{lab}}, \mathbf{x}_{\text{lab}})$ will aid our efforts to quantify the association between $Y$ and $X$. By contrast, if $\hat{f}(Z)$ is a very poor prediction of $Y$ on our population of interest, then using $(\hat{f}(\mathbf{z}_{\text{unlab}}), \mathbf{x}_{\text{unlab}})$ in addition to $(\mathbf{y}_{\text{lab}}, \mathbf{x}_{\text{lab}})$  may hinder our efforts. 
Our goal is valid quantification of the association between $Y$ and $X$, \emph{regardless of the quality of $\hat{f}(\cdot)$ on the population of interest}.

To achieve this goal, \citet{wang2020methods} propose to (Step 1') model the association between $Y$ and $\hat{f}(Z)$ using the labeled dataset $(\mathbf{y}_{\text{lab}}, \hat f(\mathbf{z}_{\text{lab}}))$, and then (Step 2') incorporate the model in Step 1' to conduct  inference between $Y$ and $X$ using the unlabeled data $( \hat{f}(\mathbf{z}_{\text{unlab}}), \mathbf{x}_{\text{unlab}})$. See Box 3. %

\begin{mybox}[floatplacement=!h,label={box:box2}]{\emph{\citet{wang2020methods}'s proposal to correct prediction-based inference.} }
\begin{enumerate}[start=1,label={Step \arabic*':},leftmargin=*]
\item Model the association between  $Y$ and $\hat{f}(Z)$ using the labeled data $(\mathbf{y}_{\text{lab}}, \hat{f}(\mathbf{z}_{\text{lab}}))$.
\item Incorporate the model from Step 2' to conduct inference on the association between $Y$ and $X$ using the unlabeled data $(\hat{f}(\mathbf{z}_{\text{unlab}}), \mathbf{x}_{\text{unlab}})$. To do this, bootstrap and analytical approaches are proposed. 
  \end{enumerate}
\end{mybox}

By contrast, \citet{angelopoulos2023prediction}  propose to de-bias the estimates obtained using the unlabeled data using information from the labeled data. In the case of estimands that are linear in $Y$, they (Step 1'') compute the difference between the  estimate of the parameter of interest obtained using 
$(\hat{f}(\mathbf{z}_{\text{lab}}), \mathbf{x}_{\text{lab}})$ and the estimate obtained using $(\mathbf{y}_{\text{lab}}, \mathbf{x}_{\text{lab}})$. 
 They then (Step 2'') correct the estimate obtained using the unlabeled dataset $(\hat{f}(\mathbf{z}_{\text{unlab}}), \mathbf{x}_{\text{unlab}})$   by this amount. See Box 4. \citet{angelopoulos2023prediction} also propose a more general framework for estimands that minimize the expectation of a general loss function, though we focus on the linear case in this paper for simplicity. 

\begin{mybox}[floatplacement=!h,label={box:box2}]{\emph{\citet{angelopoulos2023prediction}'s proposal to correct prediction-based inference (in the special case of estimands that are linear in $Y$).}}
\begin{enumerate}[start=1,label={Step \arabic*'':},leftmargin=*]
\item Compute the difference between the  estimate of the parameter of interest obtained using 
$(\hat{f}(\mathbf{z}_{\text{lab}}), \mathbf{x}_{\text{lab}})$ and the estimate obtained using $(\mathbf{y}_{\text{lab}}, \mathbf{x}_{\text{lab}})$. 
\item Correct the parameter estimate obtained using the unlabeled dataset $(\hat{f}(\mathbf{z}_{\text{unlab}}), \mathbf{x}_{\text{unlab}})$   by 
 the difference computed in Step 1''.
  \end{enumerate}
\end{mybox}

In this paper, we investigate these two proposals. In Section 2, we ask a fundamental question: what parameter is each proposal targeting? We see that the proposal of \citet{angelopoulos2023prediction} targets the parameter of interest, whereas that of \citet{wang2020methods} does not. In Sections 3 and 4, we investigate the empirical consequences of our findings from Section 2. These empirical investigations paint a clear picture: namely, that failure to target the correct parameter has substantial statistical consequences for the proposal of \citet{wang2020methods}, in the form of hypothesis tests that fail to control the Type 1 error, and confidence intervals that fail to attain the nominal coverage. The proposal of \citet{angelopoulos2023prediction} does not suffer these consequences, as it targets the correct parameter. We close with a discussion in Section 5. 

 In this paper, we use capitals to represent a random variable and lower case to represent its realization. Vectors of length equal to the number of observations, or matrices whose rows correspond to the observations, are in bold.

 \section{What parameter is each method targeting?}

For concreteness, suppose that we would have fit a linear regression model on realizations of $(Y, X)$ using least squares, had a large number of realizations of $(Y, X)$ been available. Therefore, our goal is to conduct inference on the population parameter
 \begin{equation}
 \beta^* = \argmin_{\beta} \E[(Y - \beta^\top X)^2] = \E[XX^\top]^{-1}\E[XY]. %
 \label{eq:target}
 \end{equation}

The naive method (Box 1) and the proposals of \citet{wang2020methods} and \citet{angelopoulos2023prediction} use a test statistic of the form $$T_j = \frac{\hat \beta_j - \beta^*_j}{\widehat{\SE}(\hat\beta_j)}$$ for testing or constructing confidence intervals for $\beta^*_j$. They rely on it having a known distribution, or converging in distribution to a known distribution with increasing sample size. 

However, if $\hat \beta_j \overset{p}{\not \to} \beta^*_j$ and $\widehat{\SE}(\hat\beta_j)$ goes to $0$ as $n_{\text{lab}}$ and $n_{\text{unlab}}$ increase, then $T_j$ does not converge in distribution (see Appendix A for a formal statement of this). Therefore, consistency of $\hat \beta_j$ for $\beta^*_j$ is a necessary condition for valid inference using this approach. We now investigate whether this is the case.
 
 \subsection{The general case for an arbitrary prediction model $\hat f$} \label{sec:general}

We first consider the naive approach, as defined in Box 1. Fitting a linear model with least squares would result in $\hat \beta_{\text{naive}} = (\mathbf{x}_{\text{unlab}}^\top \mathbf{x}_{\text{unlab}})^{-1}\mathbf{x}_{\text{unlab}}^\top \hat f(\mathbf{z}_{\text{unlab}})$. We can see that $\hat \beta_{\text{naive}} \overset{p}{\to} \E[XX^\top]^{-1}\E[X\hat{f}(Z)]$ as $n_{\text{unlab}}$ increases, which is \emph{not} equal to  $\beta^* = \E[XX^\top]^{-1}\E[XY]$ in general. In fact, viewing $\hat{f}(\cdot)$ as a black-box function with unknown operating characteristics in the study population, we see that  the  quantity $\E[XX^\top]^{-1}\E[X\hat{f}(Z)]$  does not even involve the response, $Y$; therefore, the parameter targeted by the naive method is not of any scientific interest.

\begin{algorithm}[t]\caption{Bootstrap  correction of \citet{wang2020methods}. {The goal is to conduct inference on the association between $Y$ and $X$}.}\label{alg:original}
\setstretch{1.15}
\textcolor{white}{X} 1. Use $(\mathbf{y}_{\text{lab}}, \hat f(\mathbf{z}_{\text{lab}}))$ to fit the ``relationship model" $Y | \hat{f}(Z) \sim K(\hat{f}(Z), \phi)$, yielding $\hat{\phi}$. \\ 
\textcolor{white}{X} 2. For $b=1,\ldots,B$: \\
\textcolor{white}{XX} 2.1. Sample unlabeled observations with replacement to obtain $\mathbf{z}_{\text{unlab}}^b$ and $\mathbf{x}_{\text{unlab}}^b$. \\
\textcolor{white}{XX} 2.2. Sample outcomes $\tilde{\mathbf{y}}^b | \hat{f}(\mathbf{z}_{\text{unlab}}^b)$ from the relationship model $K(\hat{f}(\mathbf{z}_{\text{unlab}}^b), \hat\phi)$. \\
\textcolor{white}{XX} 2.3. Use $(\tilde{\mathbf{y}}^b, \mathbf{x}_{\text{unlab}}^b)$ to fit a ``regression model" for the relationship between $Y$ and $X$, and record the coefficient estimate  $\hat{\beta}^b$ and model-based standard error $\hat{s}^b$.\\
\textcolor{white}{X} 3. Compute the point estimate $\hat{\beta} = \median \{\hat{\beta}^1, \dots, \hat{\beta}^B\}$. \\
\textcolor{white}{X} 4. Compute the  ``nonparametric'' standard error $\widehat\se(\hat{\beta}) = \sd \{\hat{\beta}^1, \dots, \hat{\beta}^B\}$. \\
\textcolor{white}{X} 5. Compute the  ``parametric'' standard error $\widehat\se(\hat{\beta}) = \median \{\hat{s}^1, \dots, \hat{s}^B\}$.
\end{algorithm}

Now, we consider the bootstrap variant of the proposal of \citet{wang2020methods}, which is introduced in Box 3. A detailed description of these proposals are presented in Algorithm 1. We take Step 2.3 to involve a least squares regression. Note that Step 2.2 of Algorithm~\ref{alg:original} involves sampling observations $\tilde{\mathbf{y}}^b$ from $K(\hat{f}(Z), \hat\phi)$ for use in fitting a regression model in Step 2.3, giving us an estimate $\hat \beta_{\text{Wang}}^{\text{bootstrap}, b} = (\mathbf{x}_{\text{unlab}}^\top \mathbf{x}_{\text{unlab}})^{-1}\mathbf{x}_{\text{unlab}}^\top \tilde{\mathbf{y}}^b$. Thus $\hat \beta_{\text{Wang}}^{\text{bootstrap}, b} \overset{p}{\to} \E[XX^\top]^{-1}\E[XK(\hat{f}(Z), \hat\phi)]$ as $n_{\text{unlab}}$ increases, where we slightly abuse notation by letting $K(\hat{f}(Z), \hat\phi)$ denote a random variable with distribution $K(\hat{f}(Z), \hat\phi)$.  In general, $\E[XX^\top]^{-1}\E[XK(\hat{f}(Z), \hat\phi)]$ is not equal to the parameter of interest $\beta^* = \E[XX^\top]^{-1}\E[XY]$. Note that both the ``parametric'' and ``non-parametric'' bootstrap corrections suffer from this issue. 

The analytic variant of the proposal of \citet{wang2020methods} adjusts the naive estimate by the coefficient of $\hat{f}(Z)$ in a regression of $Y$ onto $\hat{f}(Z)$, with an intercept, using the labeled data\footnote{The expression for $\hat \beta_{\text{Wang}}^{\text{analytical}}$ given in (\ref{eq:wanganalytical}) is implemented in \citet{wang2020methods}'s code. Their publication involves a slightly different expression for $\hat \beta_{\text{Wang}}^{\text{analytical}}$, which also does not converge to the parameter of interest $\beta^*$ for a similar reason -- we show this in Appendix B.}; that is, the estimate takes the form
\begin{equation}\label{eq:wanganalytical}
\hat \beta_{\text{Wang}}^{\text{analytical}} = \frac{\widehat{\Cov}(\mathbf{y}_{\text{lab}}, \hat{f}(\mathbf{z}_{\text{lab}}))}{\widehat{\Var}(\hat{f}(\mathbf{z}_{\text{lab}}))} \hat \beta_{\text{naive}}.
\end{equation}
Thus as  $n_{\text{lab}}$ and $n_{\text{unlab}}$ increase, 
$$
\hat \beta_{\text{Wang}}^{\text{analytical}} \overset{p}{\to} \frac{{\Cov}(Y, \hat{f}(Z))}{{\Var}(\hat{f}(Z))} \E[XX^\top]^{-1}\E[X\hat{f}(Z)],
$$
which is not again not equal to $\beta^*$ in general. 

This highlights a cause for concern about the proposals of \citet{wang2020methods}: namely, that the wrong parameter is being targeted. This calls into question whether the inference that they propose will achieve the desired statistical guarantees; we investigate this issue further in the next two sections.

 Finally, we turn to the proposal of \citet{angelopoulos2023prediction}, which is introduced in Box 4. In the case of linear regression, their estimate takes the form 
 $$
\hat \beta_{\text{Angelopoulos}} = (\mathbf{x}_{\text{unlab}}^\top \mathbf{x}_{\text{unlab}})^{-1}\mathbf{x}^\top_{\text{unlab}}\hat f(\mathbf{z}_{\text{unlab}}) + \left\{(\mathbf{x}_{\text{lab}}^\top \mathbf{x}_{\text{lab}})^{-1}\mathbf{x}^\top_{\text{lab}} \mathbf{y}_{\text{lab}} - (\mathbf{x}_{\text{lab}}^\top \mathbf{x}_{\text{lab}})^{-1}\mathbf{x}^\top_{\text{lab}}\hat f(\mathbf{z}_{\text{lab}})\right\}.
 $$
 As $n_{\text{lab}}$ and $n_{\text{unlab}}$ increase, we see that
  $$
\hat \beta_{\text{Angelopoulos}} \overset{p}{\to} \E[XX^\top]^{-1}\E[X\hat{f}(Z)] + \left\{\E[XX^\top]^{-1}\E[XY] - \E[XX^\top]^{-1}\E[X\hat{f}(Z)]\right\} = \beta^*,
 $$
so the proposal of \citet{angelopoulos2023prediction} correctly targets the desired quantity.

 \subsection{An extreme setting where all methods target the correct quantity} \label{subsec:good}

We now consider an extreme setting where the prediction model \emph{exactly} equals the true regression function: that is,   $\hat{f}(z) = \E[Y | Z = z]$. We also assume that $X$ is contained within $Z$: that is,  $Z = (X, Z^{(2)})$ for some $Z^{(2)}$. This is a reasonable assumption in practice, since the covariate of interest is likely also a predictor in the machine learning model. In this extreme setting, the naive method targets 
\begin{align*}
\E[XX^\top]^{-1}\E[X\hat{f}(Z)] &= \E[XX^\top]^{-1}\E \left[ X\E[Y|Z] \right] \\
&=\E[XX^\top]^{-1}\E[XY] \\
&= \beta^*.\end{align*}
In other words, it  targets the correct parameter of interest. 

We will now show that the bootstrap methods of \citet{wang2020methods} can similarly target the correct parameter in this extreme setting, for example, if $K(\hat{f}(Z), \hat \phi)$ is defined by fitting a  generalized additive model (GAM)
 to $(Y, \hat{f}(Z) )$ and adding mean-zero noise, as in \citet{wang2020methods}. That is,  
 \begin{equation}
 K(\hat{f}(Z), \hat \phi)\overset{d}{=}\hat{g}(\hat{f}(Z)) + \epsilon,
 \label{eq:k}
 \end{equation}
  where $\epsilon$ is mean-zero noise and $\hat{g}(\cdot)$ is the fitted GAM.

The fitted GAM $\hat{g}(\cdot)$ takes the form
 \begin{align}
\hat g(\cdot) & \approx \argmin_g \E\left[ (Y - g(\hat f(Z)))^2 \right] \label{eq:largesample} \\
& = \argmin_g \E \left[  \left(Y - g(\E[Y | Z]) \right)^2 \right] \label{eq:4} \\
& = \argmin_g \E\left[ \E\left[ \left(Y - g(\E[Y | Z]) \right)^2 | Z \right] \right] \label{eq:5}.
\end{align}
 The approximation in \eqref{eq:largesample} holds if the labeled sample size is sufficiently large. Equation \ref{eq:4} is a consequence of the extreme assumption that $\hat{f}(z) = \E[Y | Z = z]$. Equation \ref{eq:5} follows from iterated expectations.
 
 It is not hard to see that \eqref{eq:5} is minimized when $g(\E[Y | Z]) = \E[Y | Z]$, i.e.,  $\hat g(\cdot)$ is approximately the identity function. 
 Combining this with the extreme assumption that  $\hat{f}(z) = \E[Y | Z = z]$,
 \eqref{eq:k} leads to 
  \begin{equation}
 K(\hat{f}(Z), \hat \phi) \overset{d}{\approx} \E[Y | Z] + \epsilon.
 \label{eq:k2}
 \end{equation}

Now, recall from Section~\ref{sec:general} that the parameter targeted by \cite{wang2020methods} is\\
  $\E[XX^\top]^{-1}\E[XK(\hat{f}(Z), \hat \phi)]  $.
   We observe that
  \begin{align}\E[XX^\top]^{-1}\E[XK(\hat{f}(Z), \hat \phi)] &= \E[XX^\top]^{-1}\E[X\E[K(\hat{f}(Z), \hat \phi)|Z]] \label{eq:1}\\
  & \approx \E[XX^\top]^{-1}\E[X\E[Y | Z]] \label{eq:2}\\
  &= \E[XX^\top]^{-1}\E[XY] \label{eq:3}\\
  &= \beta^*.\nonumber
  \end{align}
  Here, \eqref{eq:1} and \eqref{eq:3} follow from iterated expectations since $Z=(X, Z^{(2)})$, and \eqref{eq:2} follows from \eqref{eq:k2}. 
With a large enough labeled dataset, this approximation will hold almost exactly. 

The analytical method of Wang et al. also targets the correct parameter in this setting, since the naive method targets the correct parameter and $$\frac{{\Cov}(Y, \hat{f}(Z))}{{\Var}(\hat{f}(Z))} = \frac{{\Cov}(Y, \E[Y | Z])}{{\Var}(\E[Y | Z])} = \frac{{\Var}(\E[Y | Z])}{{\Var}(\E[Y | Z])} = 1.$$

Therefore, we have seen that under a very extreme assumption that  $\hat{f}(z) = \E[Y | Z = z]$, the methods of \cite{wang2020methods} will target (nearly) the correct parameter. 
 However, in general, this assumption is not reasonable. And in fact, our goal is valid inference on the parameter $\beta^*$ regardless of (the quality of) $\hat{f}(\cdot)$.

 \section{An empirical investigation of the distribution of the test statistic}\label{subsec:se}

We consider a simple simulation setting, inspired by the  \textit{``Simulated Data: Continuous case"} section of \cite{wang2020methods}. They generate three datasets: a \emph{training dataset} consisting of realizations of $(Z,X,Y)$ used to train a machine learning model $\hat{f}(\cdot)$, a \emph{labeled dataset} consisting of realizations of $(Z,X,Y)$, and an \emph{unlabeled dataset} consisting only of realizations of $(Z,X)$; both the labeled and unlabeled datasets are used for inference\footnote{\cite{wang2020methods} refer to the labeled data set as  ``test" data, and to the unlabeled data as ``validation" data.}. They consider predictors $Z \in \mathbb{R}^{4}$ and response $Y \in \mathbb{R}$, and define the covariate $X \equiv Z_1$.  In \cite{wang2020methods}'s paper, the 
 training, labeled, and unlabeled datasets each consist of $300$ observations. Throughout this section, we keep the training sample size fixed at $300$ observations, but vary the size of the  labeled and unlabeled datasets. 

As in \cite{wang2020methods}, we generate the training, labeled, and unlabeled datasets from the same partially linear additive model $Y = \tilde \beta_0 + \tilde \beta_1 Z_1 + \sum_{j=2}^4 \tilde \beta_j g_j(Z_j) + \epsilon$. Their goal is to  conduct inference on the marginal association between $X = Z_1$ and $Y$ in a linear model. That is, their parameter of interest is  
     \begin{equation}
 \beta_1^* = \argmin_{\beta_1} \min_{\beta_0} \E[(Y - \beta_0 - \beta_1 Z_1)^2]. %
 \label{eq:target}
 \end{equation}
 Because the features are independent, we have that
   \begin{equation}
 \beta_1^*  = \argmin_{\beta_1} \min_{\beta_0, \beta_2, \beta_3, \beta_4} \E[(Y - \beta_0 - \beta_1 Z_1 - \sum_{j=2}^4 \beta_j g_j(Z_j))^2] = \tilde \beta_1.
 \label{eq:target2}
 \end{equation}
 Thus,  $\beta_1^*$ is the marginal regression coefficient of $Y$ onto $X = Z_1$, as well as the coefficient associated with $X = Z_1$ in the partially linear additive model used to generate the data. We consider two settings: one under the null ($\beta_1^* = 0$) and one under the alternative ($\beta_1^* = 1$). 

 We generate 3 training sets and fit a GAM to each training set, to obtain three fitted models $\hat f_1, \hat f_2, \hat f_3$. In each replicate of the simulation study, we generate a new labeled and unlabeled dataset as described above. Note that this differs from the simulation in \citet{wang2020methods}, which generates a new training set (and thus a different $\hat f$) in each replicate of the simulation study. We do this to focus on the properties of estimation and inference under a fixed $\hat f$, e.g. how AlphaFold would be used in practice. We perform a total of 1,000 simulation replicates.
 
To  conduct prediction-based inference on $\beta_1^*$, both \citet{wang2020methods} and \citet{angelopoulos2023prediction} rely on the claim that $(\hat{\beta}_1 - \beta_1^*) / \widehat{\se}(\hat{\beta}_1) \overset{d}{\to} N(0, 1)$. We  consider the following versions of \citet{wang2020methods}: 
\begin{enumerate}
\item \emph{Proposal of \cite{wang2020methods}, with an analytical correction.}  Apply Box 3 with the ``analytical correction" (\ref{eq:wanganalytical}) using a linear model for the regression model and a linear model for the relationship model.
\item \emph{Proposal of \cite{wang2020methods}, with a ``parametric bootstrap'' correction.} Apply Box 3 with the  ``parametric bootstrap'' correction presented in Algorithm \ref{alg:original} using a linear model for the regression model and a GAM for the relationship model.%
\item  \emph{Proposal of \cite{wang2020methods}, with a ``non-parametric bootstrap'' correction.} Apply Box 3 with the  ``non-parametric bootstrap'' correction presented in Algorithm \ref{alg:original} using a linear model for the regression model and a GAM for the relationship model.%

\end{enumerate}
We additionally consider the proposal of \citet{angelopoulos2023prediction}:
\begin{enumerate}
  \setcounter{enumi}{3}
\item  \emph{Proposal of \citet{angelopoulos2023prediction}.} Apply Box 4 using a linear model. %
\end{enumerate}
Finally, we consider the following two approaches.
\begin{enumerate}
\setcounter{enumi}{4}
\item \emph{Classical approach using only the labeled data.} Fit a linear model to $(\mathbf{y}_{\text{lab}}, \mathbf{x}_{\text{lab}})$.
\item \emph{Naive approach.} Apply Box 1 using a linear model.
\end{enumerate}

\begin{figure}[!h]
\begin{center}
\includegraphics[width=\textwidth,page=2]{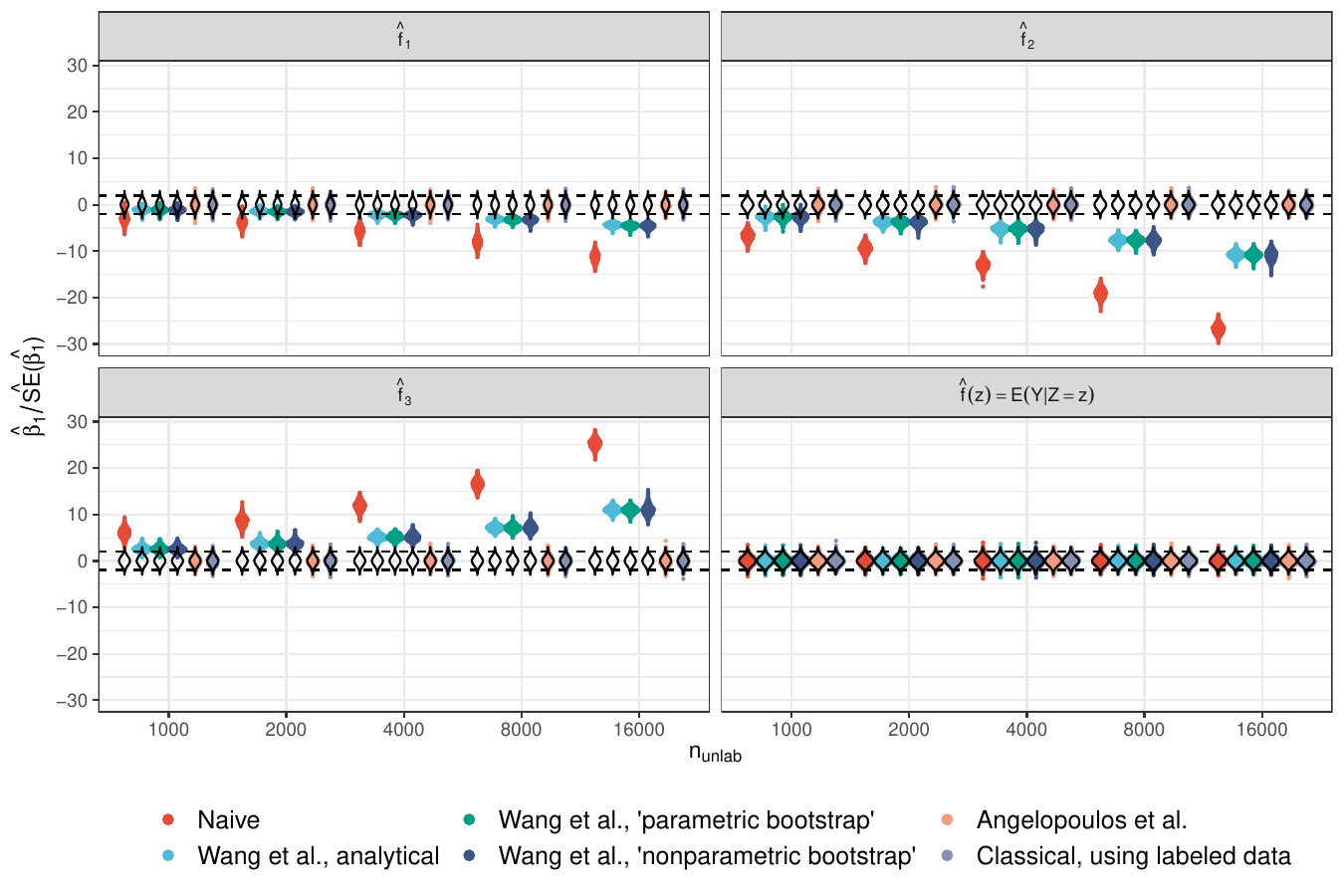}
\end{center}
\vspace{-25pt}
\caption{An examination of the distribution of $\hat{\beta}_1 / \widehat{\se}(\hat{\beta}_1)$ under $H_0: \beta_1^* = 0$. For each of four different prediction models $\hat{f}(\cdot)$ (three trained GAMs and one true regression function), we display the empirical distribution of 
$\hat{\beta}_1 / \widehat{\se}(\hat{\beta}_1)$ as  the sample sizes increase, with $n_{\text{lab}} = 0.1 n_{\text{unlab}}$. The $N(0,1)$ distribution  is shown in black. The dashed black lines show the $0.025$ and $0.975$ quantiles of this distribution. 
  The distributions of \cite{wang2020methods}'s test statistics  increasingly diverge from the $N(0,1)$ distribution as the sample sizes increase. The methods and simulation setup are described in Section 3. \label{fig:dist}}
\end{figure}
\begin{figure}[!h]
\begin{center}
\includegraphics[width=\textwidth,page=2]{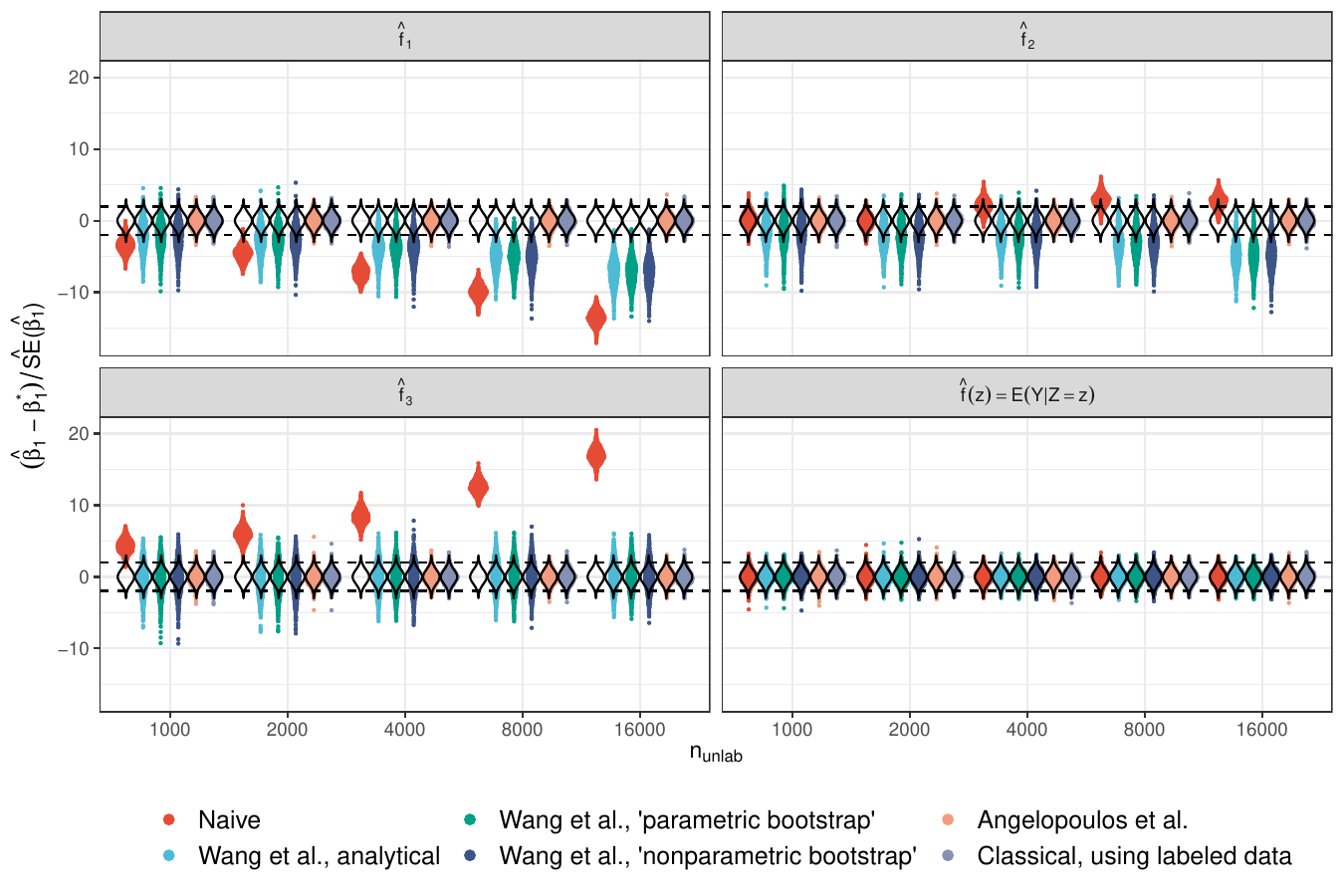}
\end{center}
\vspace{-25pt}
\caption{
An examination of the distribution of $(\hat{\beta}_1 - \beta_1^*) / \widehat{\se}(\hat{\beta}_1)$ when  $\beta_1^* = 1$. For each of four different prediction models $\hat{f}(\cdot)$ (three trained GAMs and one true regression function), we display the empirical distribution of 
$(\hat{\beta}_1 - \beta_1^*) / \widehat{\se}(\hat{\beta}_1)$ as  the sample sizes increase, with $n_{\text{lab}} = 0.1 n_{\text{unlab}}$. The $N(0,1)$ distribution  is shown in black. The dashed black lines show the $0.025$ and $0.975$ quantiles of this distribution. 
  The distributions  of  \cite{wang2020methods}'s test statistics increasingly diverge from the $N(0,1)$ distribution as the sample sizes increase. The methods and simulation setup are described in Section 3. 
  \label{fig:distalt}}
\end{figure}

In Figure \ref{fig:dist} we show the empirical distribution of $(\hat{\beta}_1 - \beta_1^*)/ \widehat{\se}(\hat{\beta}_1)$ under $H_0: \beta_1^* = 0$, for increasing sample sizes. In the first three panels, we can see that \emph{the asymptotic distribution of this test statistic for an arbitrary $\hat f$ does not follow a $N(0, 1)$ for the naive and \citet{wang2020methods} methods}. This is in line with our findings in Section 2. This is also true under the alternative $H_0: \beta_1^* = 1$, as shown in Figure \ref{fig:distalt}. On the other hand, for the method of \citet{angelopoulos2023prediction}, this statistic converges in distribution to a $N(0, 1)$ regardless of the choice of $\hat f$.

In the last panel of Figures \ref{fig:dist} and \ref{fig:distalt}, we consider the extreme setting considered in Section 2.2, in which $\hat{f}(z) = \E[Y | Z = z]$. Here, the empirical distribution of the test statistic is approximately $N(0, 1)$ for all methods.

We next performed testing and constructed confidence intervals under the assumption made by \citet{wang2020methods} and \citet{angelopoulos2023prediction} that the test statistic asymptotically follows $N(0, 1)$. We examined how the violation of this distributional assumption impacts type 1 error control and coverage in Figures \ref{fig:type1errorfixed} and \ref{fig:coveragefixed} in the Appendix. We see that \emph{the naive and \citet{wang2020methods} methods do not control the type 1 error rate or have nominal coverage for an arbitrary $\hat f$}, and they become increasingly anti-conservative as the sample sizes increase. This can be explained by the increasing discrepancy between the assumed and true distributions of  $(\hat{\beta}_1 - \beta_1^*)/ \widehat{\se}(\hat{\beta}_1)$ as the sample sizes increase, as previously seen in Figures \ref{fig:dist} and \ref{fig:distalt}. In fact, we can directly read off the type 1 error rate at level 0.05 as the  proportion of points falling outside the dashed lines in Figure 1. Similarly, coverage can be read as the proportion of points inside the dashed lines in Figure 2. 

Because the true distribution of the test statistic matches the assumed one for the method of \citet{angelopoulos2023prediction} for any $\hat f$, this method gives correct type 1 error control and coverage in general. For the same reason, under the extreme setting, all methods have correct  type 1 error control and coverage.

\section{A direct replication of the simulation study of \cite{wang2020methods}} \label{sec:simulations}

In the previous section, we considered a simulation setting that was very similar to that in \cite{wang2020methods}, but differed in that we considered the same prediction models $\hat f$ across all simulation replicates. In this section, we instead directly replicate their simulation setting (``Simulated data; continuous case''), by generating a new training dataset in each simulation replicate (resulting in a different $\hat f$ in each replicate). They considered a sample size of 300 for the training, labeled, and unlabeled datasets. In addition to replicating these results, we explore increasing the labeled and unlabeled dataset sizes.

\begin{figure}[!h]
\begin{center}
\includegraphics[width=\textwidth]{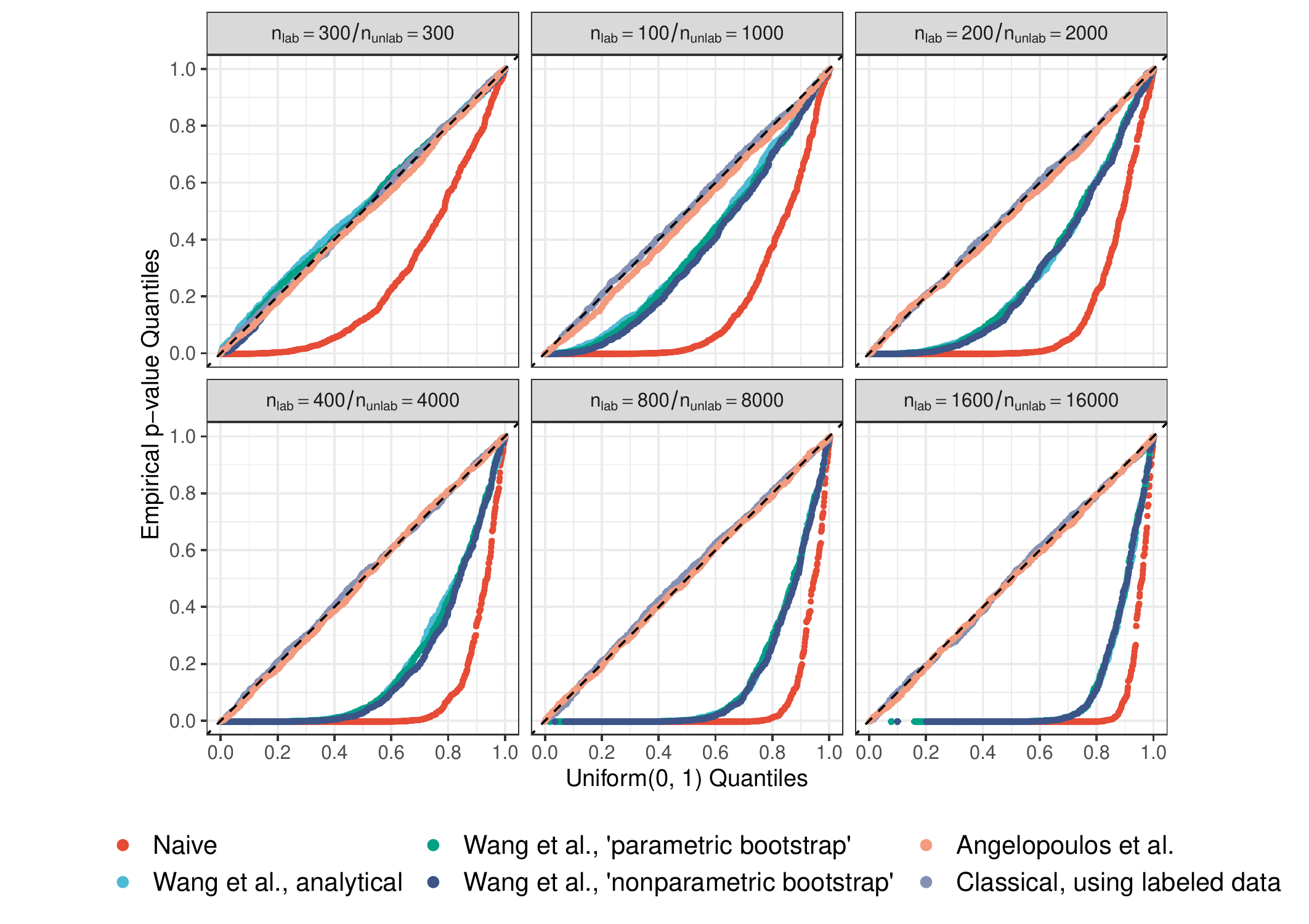}
\end{center}
\vspace{-25pt}
\caption{For data generated under $H_0: \beta_1^* = 0$, quantile-quantile plots of the p-values across simulation replicates are displayed. The methods are described in Section 3 and the simulation setup is described in Section 4.  Each panel corresponds to a different sample sizes of the labeled and unlabeled datasets used for inference. 
  The bootstrap and analytical corrections considered by \cite{wang2020methods} become increasingly anticonservative as the sample sizes increase.  The classical approach, and that of \cite{angelopoulos2023prediction}, are well-calibrated. 
\label{fig:type1error}}
\end{figure}

We first examine the type 1 error rate under the null hypothesis $H_0: \beta_1^* = 0$ using each of the approaches described in Section~\ref{subsec:se}. Quantile-quantile plots of the resulting $1,000$ p-values are shown in Figure \ref{fig:type1error}. We see that in agreement with the findings in \citet{wang2020methods}, the naive approach does not control the type 1 error rate regardless of sample size. While it appears that when $n_{\text{lab}} = n_{\text{unlab}} = 300$ (first panel of Figure \ref{fig:type1error}) the methods of \citet{wang2020methods} have uniform p-values under the null (as also reported in \citet{wang2020methods}), with other sample sizes this no longer holds and \emph{the methods of \citet{wang2020methods} fail to control the type 1 error rate}. As expected based on the previous section, \citet{angelopoulos2023prediction} controls type 1 error.  This finding is in agreement with the theoretical results in \citet{angelopoulos2023prediction}, which hold for an arbitrary prediction function $\hat{f}(\cdot)$. Also as expected, the classical method controls type 1 error.

\begin{figure}[!h]
\begin{center}
\includegraphics[width=\textwidth]{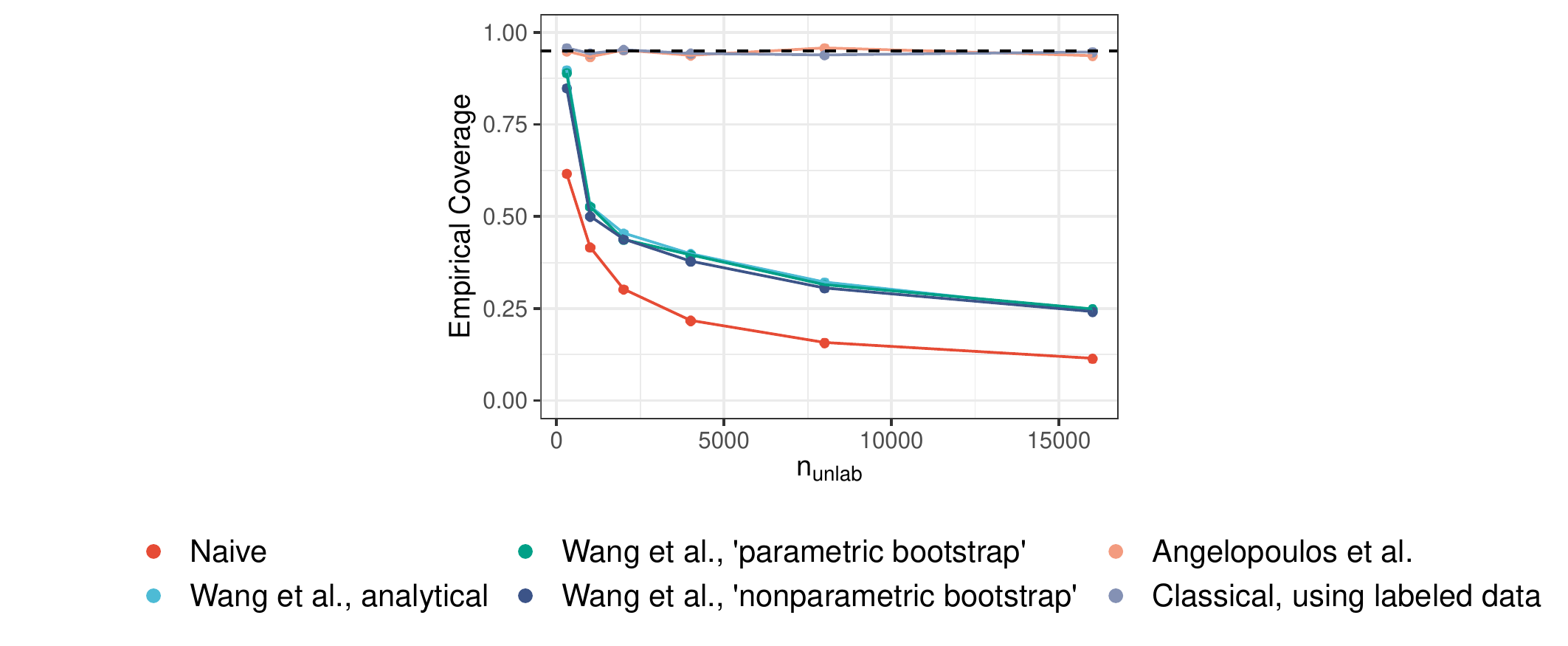}
\end{center}
\vspace{-25pt}
\caption{For data generated with $\beta_1^* = 1$, empirical coverage of 95\% confidence intervals for each method across each simulation replicate, as the labeled and unlabeled sample sizes increase, with $n_{\text{lab}} = 0.1 n_{\text{unlab}}$. The methods are described in Section 3 and the simulation setup is described in Section 4. \label{fig:coverage}}
\end{figure}

Next we examine coverage of 95\% confidence intervals under $\beta_1^* = 1$. We see from  Figure \ref{fig:coverage} that the naive approach has coverage well below the nominal level. Again, \emph{the ``corrected" proposals of \citet{wang2020methods} also fail to achieve the nominal coverage}. The problem becomes increasingly pronounced as the sample sizes of the data used for inference increase. By contrast, the classical method and the proposal of  \citet{angelopoulos2023prediction}  do achieve the nominal coverage, supporting  the theory of \citet{angelopoulos2023prediction}.

As explained in Section 2, the lack of inferential guarantees for the naive approach and the approach of \citet{wang2020methods} is a direct consequence of the fact that $\hat \beta_1 \overset{p}{\not \to} \beta^*_1$.

\section{Discussion} \label{sec:discussion}

In this paper, we found that the method of \citet{angelopoulos2023prediction} provides valid type 1 error control and coverage. 
By contrast, the methods proposed by \citet{wang2020methods} do not provide appropriate inferential guarantees in the absence of very strong assumptions: for instance, under the extreme (and unrealistic) scenario where the prediction model is the true regression function.  \emph{Under this additional assumption, the naive approach also provides valid inference.} Furthermore, we see in our simulation study that simply assuming that the prediction model $\hat f(\cdot)$ was trained on data from the same population as the labeled and unlabeled data --- an assumption made by \citet{wang2020methods} --- is not sufficient to achieve valid inference using their proposed methods.

Throughout, for simplicity we have assumed that we are conducting inference (i.e. Step 2 of Box 1) using a linear model.
However, our conclusions --- that the naive method and methods of \citet{wang2020methods} result in invalid inference, because they target the incorrect parameter --- apply much more generally. The theory in \citet{angelopoulos2023prediction} shows that their approach applies to a wide variety of settings beyond linear regression, and has valid inferential properties. 

\section*{Code Availability}
Scripts to reproduce the results in this manuscript are available at \url{https://github.com/keshav-motwani/PredictionBasedInference/}. Our code is based on the code from \citet{wang2020methods}; we thank the authors for making it publicly accessible. 

\section*{Acknowledgements}
We thank Jeff Leek and Tyler McCormick for helpful conversations. The authors gratefully acknowledge funding from \textit{NIH R01 EB026908, NIH R01 GM123993, NIH R01 DA047869, ONR N00014-23-1-2589}, a \textit{Simons Investigator Award in Mathematical Modeling of Living Systems}, and the \textit{NSF Graduate Research Fellowship Program}.

\clearpage

\bibliography{postpi_bib}
\clearpage
\begin{appendix} 

\section{Necessity of consistency for $\beta^*$}

We formally state the necessity of targeting the correct parameter in order to use the test statistic $(\hat \beta_j - \beta^*_j)/\widehat{\SE}(\hat\beta_j)$  for inference.
 \begin{lemma}
Suppose $\widehat{\SE}(\hat\beta_j) = o_p(1)$ and $\hat \beta_j \overset{p}{\not \to} \beta^*_j$. Then $(\hat \beta_j - \beta^*_j)/\widehat{\SE}(\hat\beta_j)$ does not converge in distribution.
 \end{lemma}
\begin{proof}
Suppose, for the sake of contradiction, that $(\hat \beta_j - \beta^*_j)/\widehat{\SE}(\hat\beta_j)$ converges in distribution. Then $(\hat \beta_j - \beta^*_j)/\widehat{\SE}(\hat\beta_j) = O_p(1)$. Thus
$$
\hat \beta_j - \beta^*_j = \frac{\hat \beta_j - \beta^*_j}{\widehat{\SE}(\hat\beta_j)}\widehat{\SE}(\hat\beta_j) = O_p(1)o_p(1) = o_p(1),
$$
so $\hat \beta_j \overset{p}{\to} \beta^*_j$, which is a contradiction. 
\end{proof}

\section{Lack of consistency of analytical method of \citet{wang2020methods}}

In Sections 2.1 and 2.2, we analyzed the analytical correction as implemented in the code by \citet{wang2020methods}. We now analyze the analytical correction as described in the publication, which is also not consistent for the parameter of interest in general. In the publication, the estimate obtained from the analytical correction is defined as
$$
\hat \beta_{\text{Wang}}^{\text{analytical}} = \hat \gamma_0 (\mathbf{x}_{\text{unlab}}^\top \mathbf{x}_{\text{unlab}})^{-1}\mathbf{x}_{\text{unlab}}^\top   \mathbf{1}_{n_{\text{unlab}}} + \hat \gamma_1 \hat \beta_{\text{naive}}
$$
where $$
\hat \gamma_1 = \frac{\widehat{\Cov}(\mathbf{y}_{\text{lab}}, \hat{f}(\mathbf{z}_{\text{lab}}))}{\widehat{\Var}(\hat{f}(\mathbf{z}_{\text{lab}}))} 
$$
and 
$$
\hat \gamma_0 = \hat \E[\mathbf{y}_{\text{lab}}] - \hat \gamma_1 \hat \E[\hat{f}(\mathbf{z}_{\text{lab}})].
$$
Thus as $n_{\text{lab}}$ increases, 
$$
\hat \gamma_1 \overset{p}{\to} \frac{{\Cov}(Y, \hat{f}(Z))}{{\Var}(\hat{f}(Z))} \equiv \gamma_1
$$
and
$$
\hat \gamma_0 \overset{p}{\to}  \E[Y] -  \gamma_1  \E[\hat{f}(Z)] \equiv \gamma_0.
$$
Thus as $n_{\text{lab}}$ and $n_{\text{unlab}}$ increase,
$$
\hat \beta_{\text{Wang}}^{\text{analytical}}  \overset{p}{\to}  \gamma_0 \E[XX^\top]^{-1}\E[X] +  \gamma_1 \E[XX^\top]^{-1}\E[X\hat{f}(Z)],
$$
which is not equal to $\beta^*$ in general. 

In the extreme setting (described in Section 2.2), in which $\hat{f}(z) = \E[Y | Z = z]$, we have that 
$$
\gamma_1 = \frac{{\Cov}(Y, \hat{f}(Z))}{{\Var}(\hat{f}(Z))} = \frac{{\Cov}(Y, \E[Y | Z])}{{\Var}(\E[Y | Z])} = \frac{{\Var}(\E[Y | Z])}{{\Var}(\E[Y | Z])} = 1
$$ and $$\gamma_0 = \E[Y] -  \gamma_1  \E[\hat{f}(Z)] = \E[Y] - \E[\E[Y | Z]] = 0.$$ Thus using the same argument as for the naive estimator, this analytical correction is consistent for $\beta^*$ in this extreme setting.

\section{Inferential consequences of wrong distribution}
\label{app:fixed}

In this section, we examine how violation of the distributional assumption impacts type 1 error control and coverage in the simulations in Section 3. In Figure \ref{fig:type1errorfixed}, we see that the methods proposed by \citet{wang2020methods} do not control the type 1 error rate for arbitrary $\hat{f}(\cdot)$; the situation gets worse as the sample size increases. However, the method proposed by \citet{angelopoulos2023prediction} does control the type 1 error rate. \citet{wang2020methods} controls the type 1 error rate if the machine learning model is the true regression function $\hat{f}(z) = \E[Y | Z = z]$; of course, such a perfect machine learning model is unattainable in practice.

\begin{figure}[!h]
\begin{center}
\includegraphics[width=\textwidth]{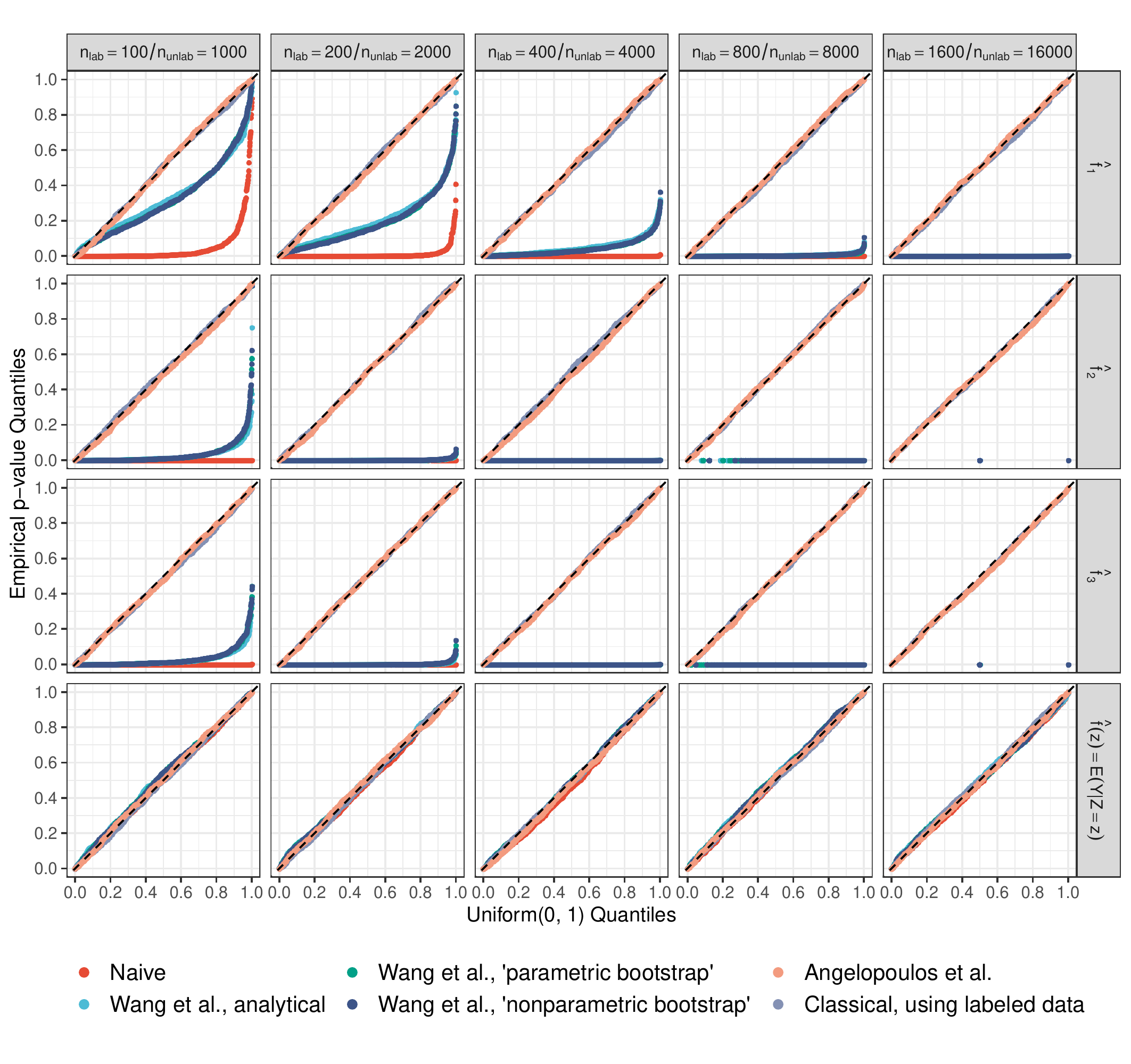}
\end{center}
\vspace{-25pt}
\caption{For labeled and unlabeled datasets generated under $H_0: \beta_1^* = 0$, quantile-quantile plots of the p-values across replicates of the modified simulation study are displayed for each of the four $\hat{f}(\cdot)$'s considered. The methods and simulation setup are described in Section~\ref{subsec:se}.  Each panel corresponds to different sample sizes of the labeled and unlabeled datasets used for inference. 
  The naive method and the bootstrap and analytical corrections considered by \cite{wang2020methods} become increasingly anticonservative as the sample sizes increases, unless the machine learning model is perfect, i.e.   $\hat{f}(z) = \E[Y | Z = z]$. \label{fig:type1errorfixed}}
\end{figure}

\begin{figure}[!h]
\begin{center}
\includegraphics[width=\textwidth]{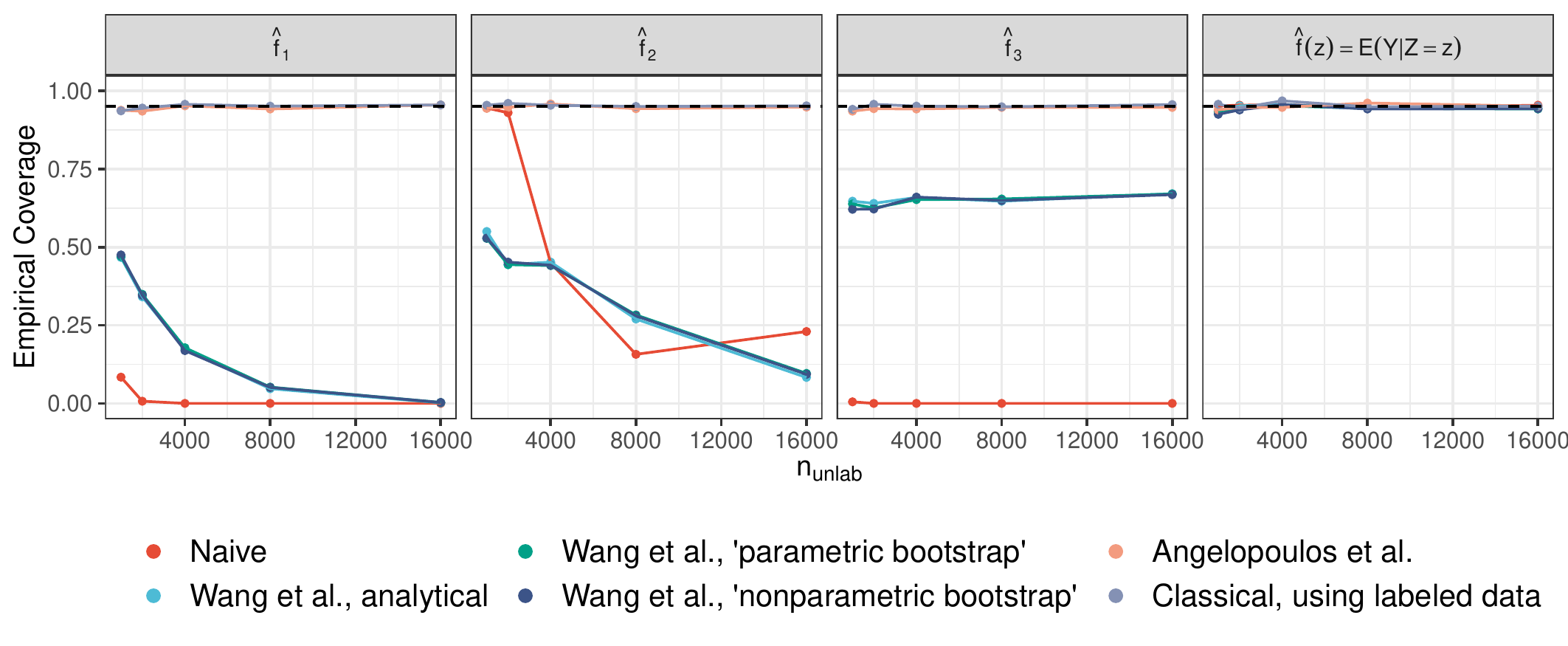}
\end{center}
\vspace{-25pt}
\caption{For labeled and unlabeled datasets generated with $\beta_1^* = 1$, empirical coverage of 95\% confidence intervals for each method across each simulation replicate, for each of the four $\hat{f}(\cdot)$'s considered, as the labeled and unlabeled sample sizes increase, with $n_{\text{lab}} = 0.1 n_{\text{unlab}}$. The methods and simulation setup are described in Section~\ref{subsec:se}. \label{fig:coveragefixed}}
\end{figure}

In Figure \ref{fig:coveragefixed}, we  see that the methods proposed by \citet{wang2020methods} do not attain the nominal coverage for an arbitrary $\hat{f}(\cdot)$,  whereas the proposal of \citet{angelopoulos2023prediction} does attain the nominal coverage. The naive method and the proposals of \citet{wang2020methods} have appropriate coverage when $\hat{f}(z) = \E[Y | Z = z]$; again, this is unrealistic in practice.

\end{appendix}

\end{document}